\documentclass[letterpaper, 10 pt, journal, twoside]{IEEEtran}

\ifCLASSINFOpdf
\else
\fi
\usepackage{amsmath,amsfonts}
\usepackage{hyperref}
\usepackage{algorithmic}
\usepackage{algorithm}
\usepackage{array}
\usepackage[caption=false,font=normalsize,labelfont=sf,textfont=sf]{subfig}
\usepackage{textcomp}
\usepackage{bm}
\usepackage{stfloats}
\usepackage{cleveref}
\usepackage{makecell}
\usepackage{url}
\usepackage{verbatim}
\usepackage{graphicx}
\usepackage[normalem]{ulem}
\usepackage{amssymb}
\usepackage{cite}
\usepackage[dvipsnames]{xcolor}
\usepackage{amsthm} 
\usepackage{threeparttable}

\newtheorem{assumption}{Assumption}
\newtheorem{lemma}{Lemma}

\newtheorem{theorem}{Theorem}
\usepackage{pifont}
\newtheorem{rem}{Remark}
\usepackage{multirow}
\usepackage{tabularray}
\usepackage{booktabs}

\begin{document}
\title{Neural-ESO: A Dual-Pathway Architecture for Provably Robust Learning-Based Control}

\author{Fan Zhang$^{1,2}$, Richie Suganda$^{1,2}$, Jinfeng Chen$^{1}$, Wenhua Liu$^{1,2}$, Hantao Fu$^{3}$, Bin Hu$^{1,2}$, Qin Lin$^{1,2}$
\thanks{Manuscript received: January, 29, 2026; Revised May, 7, Year; Accepted June, 11, 2026.}
\thanks{This paper was recommended for publication by Editor Jens Kober upon evaluation of the Associate Editor and Reviewers' comments. This material is based upon work supported by the National Science Foundation under Grant No. 2525200. \textit{(Corresponding author: Qin Lin.)}}
\thanks{$^{1}$Fan Zhang, Richie Suganda, Jinfeng Chen, Wenhua Liu, Bin Hu, and Qin Lin are with the Department of Engineering Technology, University of Houston, USA. {\tt\footnotesize qlin21@central.uh.edu}}
\thanks{$^{2}$Fan Zhang, Richie Suganda, Wenhua Liu, Bin Hu, and Qin Lin are with the Department of Electrical and Computer Engineering, University of Houston, USA}
\thanks{$^{3}$Hantao Fu is with the Department of Electrical Engineering, Rice University, USA.}
\thanks{Digital Object Identifier (DOI): see top of this page.}
}

\markboth{IEEE Robotics and Automation Letters. Preprint Version. Accepted June, 2026}
{Zhang \MakeLowercase{\textit{et al.}}: Neural-ESO: A Dual-Pathway Architecture for Provably Robust Learning-Based Control}

\maketitle

\begin{abstract}
A learning-enabled disturbance-rejection framework based on a Neural Extended State Observer (Neural-ESO) is presented in this letter. Unlike existing learning-based control methods that largely rely on the learned model once deployed, Neural-ESO adopts a dual-pathway architecture: a predictive pathway uses a neural network to provide a feedforward disturbance estimate that accelerates convergence, while a corrective pathway employs a conventional ESO to compensate prediction errors and prevent over-reliance on the neural component. Using Lyapunov theory and a small-gain analysis, we show that enforcing a Lipschitz bound on the learning component guarantees uniform ultimate boundedness of the closed-loop error dynamics. The proposed framework is validated on a quadrotor landing task subject to strong ground-effect disturbances across normal and out-of-distribution scenarios, demonstrating accuracy-robustness trade-off and greater operational reliability during training, deployment, and transfer compared with state-of-the-art baselines. {Video:{\url{https://youtu.be/KVUX0SVO-dA}}. Code and example dataset: {\url{https://github.com/fzhang327/Neural-ESO}}.}

\end{abstract}
\begin{IEEEkeywords}
Machine Learning for Robot Control, Robust/Adaptive Control, Aerial Systems: Mechanics and Control
\end{IEEEkeywords}
\IEEEpeerreviewmaketitle
\section{INTRODUCTION}
\label{sec:introduction}
\IEEEPARstart{T}{ime-varying} unknown disturbances are a fundamental challenge for high-performance control of situated robots deployed in the real world. Take an unmanned aerial vehicle (UAV) as an example: such disturbances can arise from model uncertainty (\emph{e.g.}, parameter identification errors) and external effects (\emph{e.g.}, wind and payload variations). To address these issues, a key question is how to measure the discrepancy between the known (nominal) model and the true system dynamics.

Learning-based control leverages deep neural networks (DNNs) to approximate either the full dynamics or modeling discrepancies~\cite{NeuralLander,o2022neural,bauersfeld2021neurobem}. A common strategy, referred to here as \emph{direct residual learning}, learns the discrepancy by fitting the mismatch between model-predicted and measured states. For autonomous drone landing, where ground effect acts as an unknown disturbance that can prevent landing, Neural-Lander~\cite{NeuralLander} learns and compensates this disturbance using deep learning. Despite their flexibility and minimal assumptions on disturbance structure, such approaches suffer from several limitations: (i) providing formal guarantees for learning-enabled components remains a significant challenge, (ii) their performance may degrade or even become unstable under out-of-distribution (OOD) conditions when test scenarios differ significantly from training, and (iii) stability is not ensured during training transients when the residual model has not yet converged. These issues hinder their deployment in safety-critical systems\cite{brunke2022safe,dawson2023safe}.
\begin{figure}[t]
    \centering
    \includegraphics[width=0.85\columnwidth]{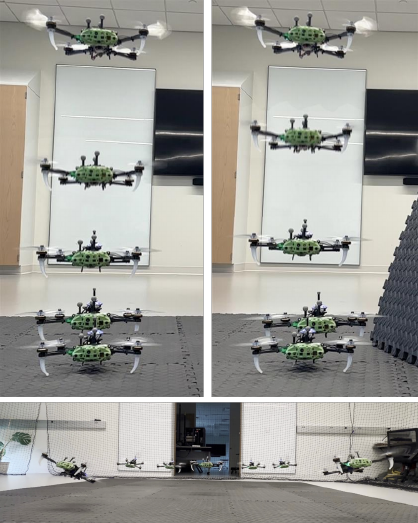}
    \caption{Experimental validation under three scenarios. (a) Top left: Normal landing, where training and testing are performed in the same setting. (b) Top right: OOD with a slope placed near the landing zone. (c) Bottom: OOD high-speed near-ground lemniscate maneuvers under nonstationary turbulence.}
    \label{fig:cover}
\end{figure}

The other line of research is rooted in control theory. For example, adaptive control provides formal guarantees for adaptation and stability. However, these approaches can be further improved if prior knowledge about the unknown disturbance is incorporated in a feedforward manner, for example through model-informed or learning-informed priors. In \cite{zhang2016active}, the authors showed that the traditional extended state observer (ESO) \cite{Han2009}, which is widely used for disturbance estimation in disturbance-rejection control, can achieve faster convergence when the disturbance can be modeled and integrated into the observer. Several recent efforts have explored combining ESOs with machine learning models \cite{liu2023cascade,li2023neural,zhang2024,yu2026parameter}. {Existing approaches often rely on relatively simple models, such as linear regression \cite{zhang2024,yu2026parameter} or radial basis function networks \cite{liu2023cascade}. More importantly, they lack a principled analysis of the risks introduced by the learning component and the corresponding requirements needed to ensure safe integration. Furthermore, they do not address how to adapt when the learned model becomes unreliable under OOD conditions.}

To systematically integrate learning into observer-based control, we address three fundamental research questions (RQs). For \emph{RQ1} (What to learn?), we adopt a disturbance-centric paradigm that approximates ESO estimates rather than direct residuals, enabling faster and model-informed predictions. For \emph{RQ2} (How to ensure stability?), we demonstrate that enforcing a Lipschitz constraint on the neural network (NN) guarantees closed-loop stability, while our dual-pathway ESO actively compensates for residual prediction errors. For \emph{RQ3} (How to generalize under OOD conditions?), we achieve safe adaptation by leveraging the combined neural and ESO estimates as the total disturbance for subsequent model retraining.

The main contributions of this letter are threefold:
\begin{enumerate}
\item We propose a disturbance-centric Neural-ESO framework with a dual-pathway architecture, enabling fast learning-based prediction while preserving robustness via online correction.
\item We establish a Lyapunov-based small-gain condition showing that a Lipschitz-bounded learning component is sufficient to guarantee Uniformly Ultimately Bounded (UUB) of the closed-loop error dynamics.
\item We demonstrate safe and effective OOD adaptation through total disturbance retraining, validated on challenging UAV landing tasks, see Fig. \ref{fig:cover}.
\end{enumerate}

The rest of the letter is organized as follows: Section \ref{sec:Preliminary} introduces the preliminaries. Section \ref{sec:Neural-ESO} presents the proposed framework. Stability analysis is demonstrated in Section \ref{sec:stability}. Hardware experiments are reported in Section \ref{sec:experiments}. Concluding remarks are in Section \ref{sec:Conslusion and future work}.
\section{Preliminaries}
\label{sec:Preliminary}
\subsection{System Dynamics and Control Law}
We consider a general robotic system with an Euler-Lagrange formulation: 
\begin{equation}
\label{eq:EL}
M\ddot{q}+C\dot{q}+{G}= u + d(q, \dot{q},t)
\end{equation}
where $q, \dot{q}, \ddot{q} \in \mathbb{R}^n$ 
denote the generalized position, velocity, and acceleration, respectively. 
$M \in \mathbb{R}^{n\times n}$, 
$C \in \mathbb{R}^{n\times n}$, and
$G \in \mathbb{R}^n$ denote the \emph{nominal} inertia matrix, Coriolis and centrifugal matrix and gravity vector;
$d(q, \dot{q}, t)$ 
represents the lumped time-varying disturbance including model uncertainties of $M$, $C$, and ${G}$ and external disturbances, and $u \in \mathbb{R}^n$ is the control input.

If $d(q, \dot{q},t)$ is neglected or compensated in \eqref{eq:EL}, the control law is designed as:
\begin{equation}
\label{eq: control law}
    u  =  {M}u_0 
    +{C} \dot{q} +{G},
\end{equation}
where the nominal control $u_0$ is designed as a reference feedforward and a Proportional-Derivative (PD) control:
\begin{equation}
\label{eq: nominal control}
\begin{split}
u_0 &=\ddot{q}^\star+k_p(q^\star-q)+k_d(\dot{q}^\star-\dot{q})\\
    \end{split},
\end{equation}
where $\ddot{q}^\star, \dot{q}^\star, q^\star \in \mathbb{R}^{n}$ represent the desired acceleration, velocity, and position of a reference joint trajectory generated by an upstream planner. $k_{p}\in \mathbb{R}^{n \times n}$ and $k_{d}\in \mathbb{R}^{n \times n}$ are the proportional and derivative gain matrices, respectively. 

\subsection{Extended State Observer Design}\label{Sec:ESO}

We use $f$ to represent the total disturbance $M^{-1}d(q, \dot{q},t)$. Then the system \eqref{eq:EL} can be rewritten into state space form:
\begin{equation}\label{eq:ESO1}
    \begin{cases}
    \begin{array}{l}
         \dot{x}_{1} = x_{2}  \\
         \dot{x}_{2} = F+Bu+f,
    \end{array}
    \end{cases}
\end{equation}
where \(x_{1}=q \); \(x_{2}=\dot{q} \); \(F={M}^{-1} (-{C} \dot{q} -{G}) \in \mathbb{R}^n\), \(B={M}^{-1}  \in \mathbb{R}^{n \times n}\), and \(f=M^{-1}d(q, \dot{q}, t)\in \mathbb{R}^n\), respectively.

By augmenting $f$ as an extended state, the augmented system can be written as:
\begin{equation}
\label{eq:augmented state space}
\begin{cases}
  \begin{array}{l}
\dot{x}_{1}=x_{2} \\
\dot{x}_{2}=F+Bu+x_{3}\\
\dot{x}_{3}=\dot{f}.
  \end{array}
\end{cases}
\end{equation}
Then, a third-order ESO to estimate the disturbance $f$ can be designed as
\begin{equation}
\label{eq: ESO}
\begin{cases}
  \begin{array}{l}
\dot{\hat{x}}_{1}=\hat{x}_{2}+\beta_{1}(x_{1}-\hat{x}_{1})\\
\dot{\hat{x}}_{2}=F+Bu+\hat{x}_{3} + \beta_{2}(x_{1}-\hat{x}_{1})\\
\dot{\hat{x}}_{3}=\beta_{3}(x_{1}-\hat{x}_{1}),
  \end{array}
\end{cases}
\end{equation}
where $\hat{x}_{1}$, $\hat{x}_{2}$, and $\hat{x}_{3}$ are the estimated $q$, $\dot{q}$, and $f$, respectively. The observer gains $\beta_{1}$, $\beta_{2}$, and $\beta_{3}$ are chosen as diagonal matrices such that the eigenvalues of the corresponding observer subsystems are assigned to $-\omega_{o}$, where $\omega_{o}$ is the observer bandwidth in the ESO design \cite{gao2003scaling}.

\section{The Proposed Neural-ESO}
\label{sec:Neural-ESO}
\begin{figure}[htbp]
    \centering
    \includegraphics[width=0.95\columnwidth]{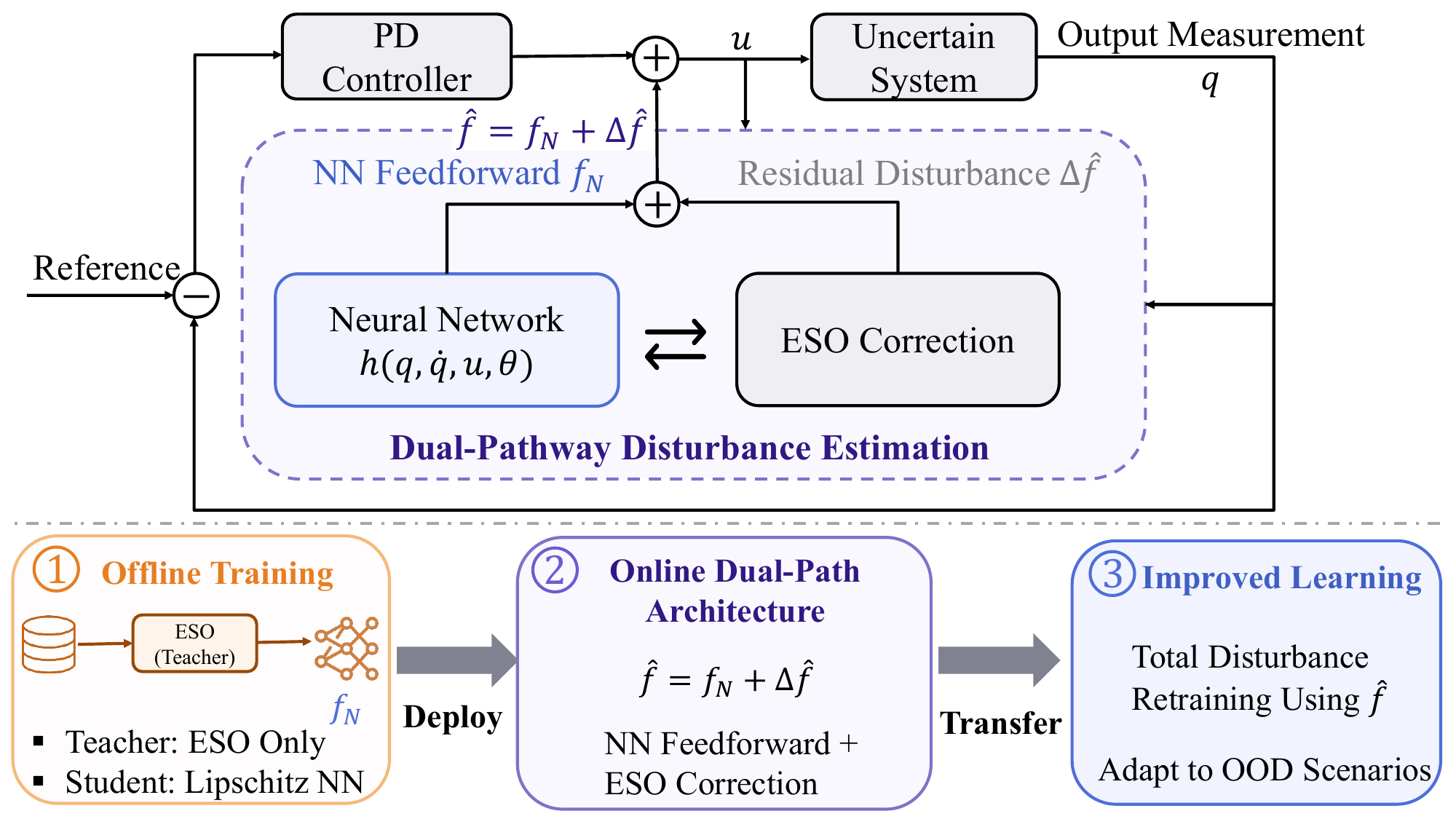}
    \caption{Proposed Neural-ESO framework with offline training (orange), online dual-pathway estimation (purple), and OOD transfer (blue).}
    \label{fig:overview}
\end{figure}

Fig.~\ref{fig:overview} shows the proposed framework. In the offline stage (orange), an ESO-only baseline is used to collect disturbance estimates for training a neural prior of the total disturbance. During online operation (purple), the neural prior provides a feedforward estimate in the predictive pathway, while a conventional ESO stays active to estimate and compensate the residual induced by prediction errors. Under OOD conditions, the Total Disturbance Retraining (TDR) block (blue) is activated. The total disturbance in the new domain, reconstructed from the neural network prediction and the ESO compensation ($\hat{f} = f_N + \Delta \hat{f}$), is safely collected and used as a training signal to retrain the neural network for performance improvement.
\label{sec:Learning based ESO}
Revisiting the observer \eqref{eq: ESO}, when the NN prior $f_N$ is incorporated, the observer only needs to estimate the residual $\Delta f = f - f_N$. In other words, the total disturbance estimate is decomposed as $\hat{f} = f_N + \Delta\hat{f}$. Accordingly, the Neural-ESO is formulated as:
\begin{equation}
\label{eq: NNESO}
\begin{cases}
  \begin{array}{l}
\dot{\hat{x}}_{1}=\hat{x}_{2}+\beta_{1}(x_{1}-\hat{x}_{1})\\
\dot{\hat{x}}_{2}=F+Bu+{f}_{N} + \Delta\hat{f}+\beta_{2}(x_{1}-\hat{x}_{1})\\
 \dot{\Delta\hat{ f}}=\beta_{3}(x_{1}-\hat{x}_{1}).
  \end{array}
\end{cases}
\end{equation}
where $\Delta\hat{f}$ becomes the new augmented state. ${f}_{N}$ is generated by a neural network parameterized by $\theta$, \emph{i.e.}, ${f}_{N}=h(z;\theta)$.
The input vector ${z}$ is defined as ${z}(t) = [x_1(t)^\top, x_2(t)^\top, u(t-T_s)^\top]^\top$, which consists of the generalized position, generalized velocity, and a one-sample delayed control input. Note that $x_1(t)$ and $x_2(t)$ are directly measurable. In contrast to \cite{NeuralLander,he2025self}, we avoid using acceleration measurements, as they are sensitive to sensor noise.

The neural network $h(z;\theta)$ is not only required to accurately approximate the underlying mapping, but also to satisfy a Lipschitz bound. This requirement constitutes one of the key insights of this letter and, to the best of our knowledge, has not been explicitly addressed in the existing literature on learning-based observers. A detailed closed-loop stability analysis is provided in Section~\ref{sec:stability}.

\subsection{The Predictive Pathway (Feedforward NN)}
\label{sec:learning}
\subsubsection{Data Collection with ESO}

As illustrated in the orange block in the overview, see Fig. \ref{fig:overview}, the offline training adopts a supervised learning setup, where a conventional ESO~\eqref{eq: ESO} serves as a \emph{teacher}. Specifically, we run the system under trajectory tracking with a properly tuned ESO \eqref{eq: ESO} to collect a dataset $\mathcal{D} = \{(z_k, y_k)\}_{k=1}^{N}$, where at each sampling instant $t_k$, the input to the neural network is recorded as
$z_k = \big[x_1(t_k)^\top,\; x_2(t_k)^\top,\; u(t_{k-1})^\top\big]^\top$,
and the corresponding label is taken from the ESO disturbance estimate, $y_k = \hat{x}_3(t_k)$, obtained from~\eqref{eq: ESO}.

\subsubsection{Constrained DNN Training with Spectral Normalization}
\label{sec: Constrained DNN}

As will be shown in the stability analysis (Section~\ref{sec:stability}), our proof relies on the disturbance derivative $\dot{\Delta f} = \dot{f} - \dot{{f}}_N$ being bounded. This requires $\dot{{f}}_N$ to be bounded, which is achieved by guaranteeing the network $h(z;\theta)$ is Lipschitz continuous.

Therefore, we adopt the constrained optimization to minimize the prediction error, subject to a hard constraint on the network's Lipschitz constant $\|h\|_{Lip} \le L_N$\cite{NeuralLander}:
\begin{equation}
\label{eq: offline_loss}
\begin{aligned}
& \underset{\theta}{\text{minimize}}
& & \frac{1}{N} \sum_{k=1}^N \| {y}_k - h({z}_k; \theta) \|^2 \\
& \text{subject to}
& & \|h\|_{Lip} \le L_N,
\end{aligned}
\end{equation}
where $L_N$ is the target Lipschitz constant for the DNN.

To enforce this constraint, we apply spectral normalization (SN) to every weight matrix in $h(z;\theta)$. SN is employed to strictly bound the network's Lipschitz constant by constraining the spectral norm of each weight matrix during optimization. We use the Adam optimizer to solve the objective in \eqref{eq: offline_loss} with SN applied. The final, optimized weights $\theta^*$ yield a feedforward estimator ${f}_{N} = h({z}; \theta^*)$. This network is then deployed in our Neural-ESO \eqref{eq: NNESO}.

\subsection{Neural-ESO Based Control Law}
We modify the nominal control $u_0$ from \eqref{eq: nominal control} to actively cancel this total estimated disturbance:
\begin{equation}
\label{eq: NNESO_u0}
u_0 =\ddot{q}^\star+k_p(q^\star-q)+k_d(\dot{q}^\star-\dot{q}) - \hat{f}.
\end{equation}
The final control law $u$ remains as defined in \eqref{eq: control law}, utilizing the updated $u_0$ from \eqref{eq: NNESO_u0}.

\section{Nonlinear Stability Analysis}
\label{sec:stability}
To prove the UUB \cite{khalil2002nonlinear} of the entire system, we employ a composite Lyapunov analysis for the full error state $E = [e_{state}^\top, \tilde{e}_{obs}^\top]^\top$. This approach is necessitated by the coupling between the tracking and observation subsystems. The analysis proceeds by: (1) deriving the error dynamics for each subsystem; (2) formally bounding the interconnection gain, which is shown to be dependent on the Lipschitz property of the trained NN; and (3) satisfying the resulting small-gain condition to prove stability.

\subsection{Tracking Error Dynamics}
\label{sec:tracking_dynamics}
We first derive the closed-loop dynamics of the tracking error $e = q^\star - q$. To analyze the closed-loop performance, we substitute the Neural-ESO control law \eqref{eq: control law} into the robot dynamics \eqref{eq:EL}:
\begin{equation}
\begin{aligned}
M\ddot{q} + C\dot{q} + G &= \left[ M(u_0) + C\dot{q} + G \right] + d(q, \dot{q},t) \\
M\ddot{q} &= M(\ddot{q}^\star + k_p e + k_d \dot{e} - \hat{f}) + d(q, \dot{q},t) \\
-M\ddot{e} &= M(k_p e + k_d \dot{e} - \hat{f}) + d(q, \dot{q},t).
\end{aligned}
\end{equation}
Since $f = M^{-1}d(q, \dot{q},t)$, we get the final error dynamics:
\begin{equation}
\label{eq: final_error_dyn}
\begin{aligned}
\ddot{e} + k_d \dot{e} + k_p e &= \hat{f} - f \\
&= ({f}_{N} + \Delta\hat{f}) - ( {f}_{N} + \Delta f ) \\
&= \Delta\hat{f} - \Delta f.
\end{aligned}
\end{equation}
We define the residual estimation error as $\tilde{f} \triangleq \Delta f - \Delta\hat{f}$. The tracking error $e$ is now driven only by $\tilde{f}$.
This second-order dynamic can be written in state-space form. Let the tracking error state be $e_{state} = [e^\top, \dot{e}^\top]^\top \in \mathbb{R}^{2n}$. The error dynamics in \eqref{eq: final_error_dyn} will be:
\begin{equation}
\label{eq:e_system}
    \dot{e}_{state} = A_e e_{state} + B_e \tilde{f},
\end{equation}
where $\tilde{f} \in \mathbb{R}^n$ is the residual estimation error vector (the input) and $A_e, B_e$ are block matrices:
\begin{equation}
A_e = 
\begin{bmatrix}
0 & I \\
-k_p & -k_d
\end{bmatrix}
, \quad
B_e = 
\begin{bmatrix}
0 \\ -I
\end{bmatrix},
\end{equation}
where  $I, 0 \in \mathbb{R}^{n \times n}$, $k_p$ and $k_d$ are diagonal positive-definite gain matrices. Since $A_e$ is Hurwitz. This subsystem is Input-to-State Stable (ISS) with respect to the input $\tilde{f}$ \cite{khalil2002nonlinear}.

\subsection{Observer Error Dynamics}
Similarly, we can derive the observer error dynamics from \eqref{eq:augmented state space} and \eqref{eq: NNESO}. Let the observer error states be:
\begin{equation}
    \tilde{x}_1 = x_1 - \hat{x}_1, \quad \tilde{x}_2 = x_2 - \hat{x}_2, \quad \tilde{f} =  \Delta f -\Delta\hat{f}.
\end{equation}
Note that $\tilde{f}$ is defined consistent with \eqref{eq: final_error_dyn}. Taking the time-derivative and substituting \eqref{eq:augmented state space} (for $x_3 = \Delta f + {f}_{N}$) and \eqref{eq: NNESO}:
\begin{equation}
\label{eq:error_deriv_mimo}
\begin{cases}
    \dot{\tilde{x}}_1  = x_2 - (\hat{x}_2 + \beta_{1} \tilde{x}_1) = \tilde{x}_2 - \beta_{1} \tilde{x}_1 \\
    \dot{\tilde{x}}_2 = (F+Bu+f) - (F+Bu+{f}_{N}+\Delta\hat{f}+\beta_{2} \tilde{x}_1) \\
    \quad ~= (f - {f}_{N} - \Delta\hat{f}) - \beta_{2} \tilde{x}_1 = \tilde{f} - \beta_{2} \tilde{x}_1 \\
    \dot{\tilde{f}}  = \dot{\Delta f} - \dot{\Delta\hat{f}} = \dot{\Delta f} - (\beta_3 \tilde{x}_1)
\end{cases}.
\end{equation}
Let the total observer error state be $\tilde{e}_{obs} = [\tilde{x}_1^\top, \tilde{x}_2^\top, \tilde{f}^\top]^\top \in \mathbb{R}^{3n}$. The dynamics can be written in state space form:
\begin{equation}
\label{eq:e_tilde_system}
    \dot{\tilde{e}}_{obs} = A_{obs} \tilde{e}_{obs} + B_{obs}  \dot{\Delta f},
\end{equation}
where $\dot{\Delta f} \in \mathbb{R}^n$ is the input, and $A_{obs}, B_{obs}$ are block matrices:
\begin{equation}
A_{obs} = 
\begin{bmatrix}
-\beta_1 & I & 0 \\
-\beta_2 & 0 & I \\
-\beta_3 & 0 & 0
\end{bmatrix}, 
\quad
B_{obs} = 
\begin{bmatrix}
0 \\ 0 \\ I
\end{bmatrix}.
\end{equation}
Since $A_{obs}$ is Hurwitz by design through pole placement, this subsystem is also ISS.
\begin{rem}\label{rem:motivation_coupling} Incorporating a prior (e.g., a neural network) has proven beneficial~\cite{zhang2016active}. In an ideal case where $f_N$ is perfect and $\dot{\Delta f}=0$, the error systems~\eqref{eq:e_system}–\eqref{eq:e_tilde_system} are autonomous and the Hurwitz property of $A_e$ and $A_{\mathrm{obs}}$ ensures global asymptotic stability. In practice, $\dot{\Delta f}\neq 0$ due to residual errors and time-varying disturbances, motivating the following UUB analysis.
\end{rem}

\subsection{Coupling Quantification for Controller and Observer}
The subsequent stability analysis hinges on quantifying the coupling between the tracking and observer subsystems. In particular, the observer error dynamics are driven by the residual derivative $\dot{\Delta f}$, while $\dot{\Delta f}$ depends on the closed-loop signals through the network input dynamics. This creates an interconnection between the tracking error $e_{\mathrm{state}}$ and the observer error $\tilde e_{\mathrm{obs}}$ via $\dot{\Delta f}$. Therefore, we establish an explicit bound on $\|\dot{\Delta f}\|$ in terms of $(e_{\mathrm{state}},\tilde e_{\mathrm{obs}})$, which yields the corresponding interconnection gains that will be used in the subsequent small-gain analysis.

We first state the mild boundedness assumptions for the system and reference trajectories.

\begin{assumption}\label{asm1_vec}
The derivative of the true disturbance is bounded: $\left\| \dot{f} \right\| \leq L_1$.
\end{assumption}

\begin{assumption}\label{asm2_vec}
The derivative of the control input is bounded: $\left\| \dot{u}(t) \right\| \leq L_u$.
\end{assumption}

\begin{rem}\label{rem:bounded_control}
In our physical implementation, this bounded-control assumption is satisfied by saturating the increment of the control input between consecutive sampling steps.
\end{rem}

\begin{assumption}\label{asm3_vec}
The desired reference trajectory and its derivatives ($q^\star, \dot{q}^\star, \ddot{q}^\star$) are bounded.
\end{assumption}

\begin{lemma}\label{lem1} Let the neural network $f_N(z)$ be $L_N$-Lipschitz. Under Assumptions \ref{asm1_vec}-\ref{asm3_vec}, the residual derivative $\|\dot{\Delta f}\|$ is linearly bounded by:
\begin{equation}
    \|\dot{\Delta f}(t)\| \leq \gamma_e \|e_{state}\| + \gamma_{\tilde{e}} \|\tilde{e}_{obs}\| + C_f,
\end{equation}
where $\gamma_e, \gamma_{\tilde{e}} > 0$ are interconnection gains proportional to $L_N$, and $C_f > 0$ is a constant.
\end{lemma}

\begin{proof}
The residual derivative is defined as $\dot{\Delta f} = \dot{f} - \dot{f}_N$. Using the triangle inequality and Assumption \ref{asm1_vec}:
\begin{equation}
    \|\dot{\Delta f}\| \leq \|\dot{f}\| + \|\dot{f}_N\| \leq L_1 + \|\dot{f}_N\|,
\end{equation}
Applying the chain rule and the $L_N$-Lipschitz property of the SN-constrained network:
\begin{equation}
    \left\lVert\dot{f}_{N}(z)\right\rVert \leq L_N \left\lVert\dot{z}(t)\right\rVert.
\end{equation}
The input derivative vector is $\dot{z}(t) = [\dot{x}_1^\top, \dot{x}_2^\top, \dot{u}(t-T_s)^\top]^\top$. We analyze its components to establish a bound dependent on the error states.
First, $\|\dot{u}\| \le L_u$ by Assumption \ref{asm2_vec}.
Second, $\dot{x}_1 = \dot{q} = \dot{q}^\star - \dot{e}$.
Third, for $\dot{x}_2 = \ddot{q}$, we substitute the closed-loop error dynamics \eqref{eq: final_error_dyn} into $\ddot{q} = \ddot{q}^\star - \ddot{e}$:
\begin{equation}
\label{eq:x2_expansion}
    \dot{x}_2 = \ddot{q}^\star - (-k_d \dot{e} - k_p e - \tilde{f}) = \ddot{q}^\star + k_d \dot{e} + k_p e + \tilde{f}.
\end{equation}
\eqref{eq:x2_expansion} explicitly reveals the linear dependence of $\dot{z}$ on the error states. Applying the triangle inequality and consistent matrix norms, we bound $\|\dot{z}\|$:
\begin{equation}
\begin{split}
    \|\dot{z}\| \le & \underbrace{(L_u + \|\dot{q}^\star\| + \|\ddot{q}^\star\|)}_{C_z}
                    + \lambda_c \|e_{state}\| + \|\tilde{e}_{obs}\|,
\end{split}
\end{equation}
where $\lambda_c$ is obtained from control gain that satisfy $\|k_p e\|+ \|k_d \dot{e}\|\leq\lambda_c \|e_{state}\|$.

Substituting this bound back into the expression for $\|\dot{\Delta f}\|$:
\begin{equation}
    \|\dot{\Delta f}\| \le L_1 + L_N ( C_z + \lambda_c\|e_{state}\| + \|\tilde{e}_{obs}\|).
\end{equation}
The proof is concluded by defining $\gamma_e =  L_N \lambda_c$, $\gamma_{\tilde{e}} = L_N$, and $C_f = L_1 + L_N C_z$.
\hfill {$\blacksquare$}
\end{proof} 
 
\begin{rem}
This lemma shows that the residual derivative $\dot{\Delta f}$ is influenced by the tracking and observer errors through the network input dynamics, and it quantifies this influence via $\gamma_e$ and $\gamma_{\tilde e}$. Notably, both gains scale linearly with the network Lipschitz constant $L_N$, which will be used in the subsequent small-gain analysis.
\end{rem}

\subsection{Composite Stability Proof (UUB)}
\label{sec:composite_stability}

With the interconnection gain analysis in Lemma \ref{lem1}, we now prove the stability of the entire closed-loop system using a composite Lyapunov analysis.

\begin{theorem}\label{thm:Track_UUB}
(Closed-Loop UUB Stability). Given Assumptions \ref{asm1_vec}-\ref{asm3_vec}, and the controller \eqref{eq: control law} utilizing the $L_N$-Lipschitz network $f_N(z)$, the entire closed-loop error state $E = [e_{state}^\top, \tilde{e}_{obs}^\top]^\top$ is UUB.
\end{theorem}
\begin{proof}
We define a composite Lyapunov function candidate as the sum of two quadratic forms for the tracking and observation subsystems:
\begin{equation}
    V(E) = V_e(e_{state}) + V_{obs}(\tilde{e}_{obs}).
\end{equation}

\textbf{1. Tracking Subsystem ($V_e$):}
Let $V_e = e_{state}^\top P_e e_{state}$, where $P_e=P_e^\top > 0$ is the solution to the Lyapunov equation $A_e^\top P_e + P_e A_e = -Q_e$ for a chosen $Q_e=Q_e^\top > 0$. The time derivative of $V_e$ along the trajectories of \eqref{eq:e_system} is:
\begin{equation}
\begin{aligned}
    \dot{V}_e &= (A_e e_{state} + B_e \tilde{f})^\top P_e e_{state} \\
             &+ e_{state}^\top P_e (A_e e_{state} + B_e \tilde{f}) \\
             &= e_{state}^\top (A_e^\top P_e + P_e A_e) e_{state} + 2 e_{state}^\top P_e B_e \tilde{f}.
\end{aligned}
\end{equation}
By substituting the Lyapunov equation $A_e^\top P_e + P_e A_e = -Q_e$ into the first term and applying the Rayleigh-Ritz theorem ($e_{state}^\top Q_e e_{state} \ge \lambda_{\min}(Q_e)\|e_{state}\|^2$), we can bound the dissipation. Applying the Cauchy-Schwarz inequality to the second term, we obtain the upper bound:
\begin{equation}
    \dot{V}_e \le - \lambda_{\min}(Q_e) \|e_{state}\|^2 + 2 \|e_{state}\| \|P_e B_e\| \|\tilde{f}\|.
\end{equation}
Since $\|\tilde{f}\| \le \|\tilde{e}_{obs}\|$, we apply Young's inequality ($2ab \le \delta_e a^2 + \frac{1}{\delta_e} b^2$ for any $\delta_e > 0$) to the cross-term:
\begin{equation}
    2 \|e_{state}\| (\|P_e B_e\| \|\tilde{e}_{obs}\|) \le \delta_e \|e_{state}\|^2 + \frac{1}{\delta_e} \|P_e B_e\|^2 \|\tilde{e}_{obs}\|^2
\end{equation}
Substituting this back, we obtain the ISS-gain for $V_e$:
\begin{equation}
\label{eq:vdot_e}
    \dot{V}_e \le - (\lambda_{\min}(Q_e) - \delta_e) \|e_{state}\|^2 + c_e \|\tilde{e}_{obs}\|^2,
\end{equation}
where $c_e = \|P_e B_e\|^2 / \delta_e$.

\textbf{2. Observer Subsystem ($V_{obs}$):}
Similarly, let $V_{obs} = \tilde{e}_{obs}^\top P_{obs} \tilde{e}_{obs}$ with $A_{obs}^\top P_{obs} + P_{obs} A_{obs} = -Q_{obs}$. Its time derivative along \eqref{eq:e_tilde_system} is:
\begin{equation}
\begin{aligned}
    \dot{V}_{obs} \le - \lambda_{\min}(Q_{obs}) \|\tilde{e}_{obs}\|^2 + 2 \|\tilde{e}_{obs}\| \|P_{obs} B_{obs}\| \|\dot{\Delta f}\|.
\end{aligned}
\end{equation}
Let $\lambda_o = \lambda_{\min}(Q_{obs})$ and define $c_{pb} \triangleq \|P_{obs} B_{obs}\|$. Substituting the bound for $\dot{\Delta f}(t)$ from Lemma \ref{lem1} yields:
\begin{equation}
\label{eq:vdot_obs_expanded}
\begin{aligned}
    \dot{V}_{obs} &\le - \lambda_o \|\tilde{e}_{obs}\|^2 \\
    &\quad + 2 c_{pb} \|\tilde{e}_{obs}\| (\gamma_e \|e_{state}\| + \gamma_{\tilde{e}} \|\tilde{e}_{obs}\| + C_f) \\
    &= - (\lambda_o - 2 c_{pb} \gamma_{\tilde{e}}) \|\tilde{e}_{obs}\|^2 \\
    &\quad + (2 c_{pb} \gamma_e) \|\tilde{e}_{obs}\|\|e_{state}\| + (2 c_{pb} C_f) \|\tilde{e}_{obs}\|.
\end{aligned}
\end{equation}
We apply Young's inequality to the two remaining cross-terms with constants $\delta_o, \delta_c > 0$:
\begin{align*}
    (2 c_{pb} \gamma_e) \|\tilde{e}_{obs}\|\|e_{state}\| &\le \delta_o \|\tilde{e}_{obs}\|^2 + \frac{(c_{pb} \gamma_e)^2}{\delta_o} \|e_{state}\|^2 \\
    (2 c_{pb} C_f) \|\tilde{e}_{obs}\| &\le \delta_c \|\tilde{e}_{obs}\|^2 + \frac{(c_{pb} C_f)^2}{\delta_c}.
\end{align*}
Substituting these into \eqref{eq:vdot_obs_expanded} and grouping terms yields:
\begin{equation}
\label{eq:vdot_obs}
\begin{aligned}
    \dot{V}_{obs} \le & - (\lambda_o - 2 c_{pb} \gamma_{\tilde{e}} - \delta_o - \delta_c) \|\tilde{e}_{obs}\|^2 \\
                     & + c_o \|e_{state}\|^2 + C_o,
\end{aligned}
\end{equation}
where $c_o = (c_{pb} \gamma_e)^2 / \delta_o$ is the coupling gain from $e_{state}$ and $C_o = (c_{pb} C_f)^2 / \delta_c$ is a bounded constant.

\textbf{3. Composite System ($\dot{V}$):}
Combining the results from \eqref{eq:vdot_e} and \eqref{eq:vdot_obs}, the derivative of the composite Lyapunov function $V = V_e + V_{obs}$ is:
\begin{equation}
\label{eq:vdot_final}
\begin{aligned}
    \dot{V} \le & \left( - (\lambda_e - \delta_e) + c_o \right) \|e_{state}\|^2 \\
                 & + \left( - (\lambda_o - \epsilon_o) + c_e \right) \|\tilde{e}_{obs}\|^2 + C_o,
\end{aligned}
\end{equation}
where $\lambda_e = \lambda_{\min}(Q_e)$ and $\epsilon_o = (2 c_{pb} \gamma_{\tilde{e}} + \delta_o + \delta_c)$.

Inequality~\eqref{eq:vdot_final} explicitly represents a Lyapunov-based small-gain condition for the interconnected system. The terms $c_e$ and $c_o$ characterize the fixed interconnection gains, while $\lambda_e$ and $\lambda_o$ quantify the intrinsic dissipation provided by the control and observer error dynamics, respectively. 

The small-gain condition requires that the stabilizing dissipation of each subsystem dominates the energy injected through the interconnection, which leads to the inequalities
\begin{equation}
    \lambda_e > \delta_e + c_o \quad \text{and} \quad \lambda_o > \epsilon_o + c_e
\end{equation}
These inequalities impose explicit lower bounds on the required stabilization gains. Since $Q_e$ and $Q_{\mathrm{obs}}$ depend on the design parameters $(k_p,k_d,\beta_k)$, the gains can be selected to satisfy the above bounds without requiring either high-gain control or high-gain observer. Define $\alpha_e = \lambda_e - \delta_e - c_o > 0$ and $\alpha_o = \lambda_o - \epsilon_o - c_e > 0$. Then, the derivative of the composite Lyapunov function satisfies
\begin{equation}
    \dot{V} \le - \alpha_e \|e_{state}\|^2 - \alpha_o \|\tilde{e}_{obs}\|^2 + C_o \le - \alpha \|E\|^2 + C_o,
\end{equation}
where $\alpha = \min(\alpha_e, \alpha_o) > 0$. This implies that $\dot{V}$ is negative whenever $\|E\| > \sqrt{C_o / \alpha}$, which proves that the entire closed-loop error state $E$ is UUB. \hfill {$\blacksquare$}
\end{proof}

\begin{rem}
The small-gain inequalities $\lambda_e > \delta_e + c_o$ and $\lambda_o > \epsilon_o + c_e$ explicitly connect the neural network Lipschitz constant $L_N$ to the interconnection strength of the closed loop. By decreasing $L_N$, the required stabilization gains for both the controller and the observer are reduced. Thus, $L_N$ acts as a tuning knob that trades learning expressiveness for robustness, enabling UUB stability with moderate gains.

\end{rem}

\begin{rem}
{\textbf{Tuning guideline:}  
We adopt a two-stage tuning strategy: first, select $L_N$ based on the desired balance between accuracy and robustness; then, tune the observer bandwidth $\omega_o$ and controller gains ($k_p$, $k_d$) to compensate for the bounded residual disturbance. This sequential design ensures that the ESO operates within its compensation capability, consistent with the small-gain stability condition.}
\end{rem}
\section{EXPERIMENTS}
\label{sec:experiments}

We use a quadrotor UAV as the experimental platform due to its sensitivity to aerodynamic disturbances (\emph{e.g.}, ground effect), making altitude control challenging. To ensure fair comparison, the same controller is used for the $x$--$y$ axes, and performance is evaluated on the $z$ (altitude) axis. A base neural network pre-trained on flat-terrain landing is used in all experiments. Experiment~1 evaluates in-distribution performance, Experiment~2 tests robustness under OOD conditions, and Experiment~3 studies adaptation via retraining with new domain data.
\subsection{Experimental System Dynamics}
The dynamics of the quadrotor are described by the standard Newton-Euler equations \cite{chen2024quadrotor}:
\begin{subequations}
\begin{align}
    \dot{{P}}^I &= {V}^I, & m_v \dot{{V}}^I &= m_v {g}^I + {R}_{IB}{F}^B + {d}^I \label{eq_dyn_vel_dot}, \\
    \dot{{R}}_{IB} &= {R}_{IB}{\Omega}_{\times}^B, & {I}_{{v}} \dot{{\Omega}}^B &= -{\Omega}_{\times}^B {I}_{{v}}{\Omega}^B + {M}^B, \label{eq_dyn_omega_dot}
\end{align}
\end{subequations}
where the translation dynamics are described in \eqref{eq_dyn_vel_dot} and the rotation dynamics in \eqref{eq_dyn_omega_dot}. \({P}^I=[X, Y, Z]^\top\), \({V}^I=[V_x, V_y, V_z]^\top\), and \({g}^I=[0, 0, -g]^\top\) are the position, velocity, and gravity vectors in inertial frame. \({\Omega}^B=[p, q, r]^\top\) is the angular velocity vector in body frame. \({R}_{IB}\) represents the rotation matrix from body to inertial frame; \({\Omega}^B_{\times}\) denotes the skew-symmetric matrix; \(m_v\) and \({I}_{{v}}\) represent the mass and inertial matrix. The terms ${F}^B$ and ${M}^B$ are the total control force and moment vectors from the actuators. ${d}^I$ is the lumped disturbance, including model uncertainty and external disturbance such as ground effect.

Our controller design focuses on the translational dynamics \eqref{eq_dyn_vel_dot}. The mapping is established by setting the generalized coordinates $q = {P}^I$. This specialization yields the following equivalences: $M(q) = m_v I$, $C(q, \dot{q}) = 0$, $G(q) = -m_v {g}^I$, the control input $u = {R}_{IB}{F}^B$, and the disturbance $d(q, \dot{q},t) = {d}^I$.
\subsection{Experimental Setup and Baselines}
The experimental platform consists of a ModalAI Seeker drone {equipped with a VOXL Flight Deck running the PX4 autopilot firmware}. An OptiTrack motion capture system provides ground-truth position and attitude estimates. The {outer-loop} control pipeline operates {at a frequency of 50 Hz} on a ground station computer (Intel i9 processor), communicating offboard control commands via MAVSDK-Python.

The Neural-ESO outer-loop controller (Section~\ref{sec:Learning based ESO}) first computes the desired acceleration $a_d$, then forms the desired force vector $F_d = m_v(a_d - g^I)$, which yields the required total thrust $T_d = \|F_d\|$, following the commonly used pipeline in~\cite{lee2010geometric}.
 A desired rotation matrix $R_d$ is then constructed such that its $z$-axis aligns with $F_d$ and its $x-y$ axes are oriented according to a desired yaw angle $\psi_d$. The final desired roll ($\phi_d$) and pitch ($\theta_d$) angles are extracted from $R_d$.
 {These attitude references, along with the desired yaw $\psi_d$ and thrust $T_d$, are mapped directly to the PX4 offboard attitude setpoint MAVLink messages.}

\subsection{Data Collection and Training Details}
We implement the network in PyTorch. Training data is collected under a conventional ESO \eqref{eq: ESO} following the supervised setup in Section~\ref{sec:learning}. The input vector is ${z}_k = [q_k, \dot{q}_k, T_{k-1}, \phi_k, \theta_k]$, including position, velocity, thrust, and attitude angles. To avoid noise amplification, we do not use acceleration measurements \cite{NeuralLander,he2025self}. Both inputs and labels are normalized using a z-score scaler, and the dataset is split into 80\% training and 20\% validation. The DNN has 4 fully connected layers with ReLU activations (5 $\to$ 20 $\to$ 25 $\to$ 10 $\to$ 1), trained using Adam ($10^{-3}$) and MSE loss for {100} epochs.
\subsection{Experiment 1: Normal Landing Subject to Ground Effect}
\label{sec:exp_hover}
As shown in Fig.~\ref{fig:cover}(a), the first experiment evaluates normal landing under ground effect following~\cite{NeuralLander}, with training and testing conducted in the same environment. The results are shown in Fig.~\ref{fig:hover_pos}. Due to the protective frame, the minimum achievable altitude is approximately $0.1$~m, which explains the reference level in all plots.

\begin{figure}[htbp]
    \centering
    \includegraphics[width=0.85\columnwidth]{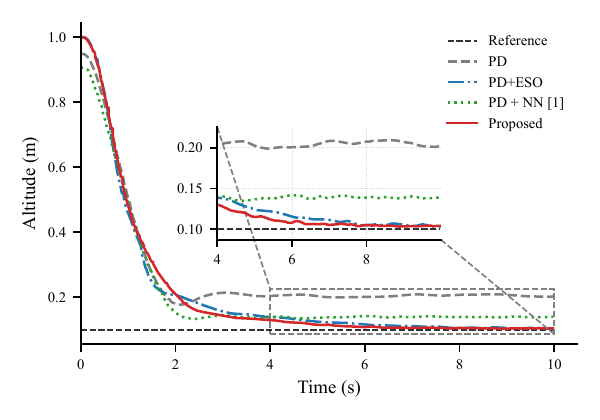}
    \caption{{Normal landing performance comparison: both PD+ESO and proposed method achieve small steady-state errors, while the proposed method achieves a lower steady-state error and faster settling time.}}
    \label{fig:hover_pos}
\end{figure}

\begin{table}[htbp]
\centering
\caption{Statistical Analysis of Steady-State Tracking Error for Normal Landing under Ground Effect (5 Trials)}
\label{tab:nominal_tracking}
\begin{tabular}{lcccc}
\toprule
\textbf{Metric} & \textbf{PD} & \textbf{PD+NN [1]} & \textbf{PD+ESO} & \textbf{Proposed} \\
\midrule
Mean [cm] & 10.95 & 2.49 & 1.84 & \textbf{0.37} \\
Std Dev [cm] & 0.57 & 0.15 & 1.89 & \textbf{0.12} \\
\bottomrule
\end{tabular}
\end{table}
{The statistical results of steady-state tracking error are summarized in Table~\ref{tab:nominal_tracking}. The nominal PD controller exhibits a large mean error of $10.95$~cm. The PD+NN method~\cite{NeuralLander} reduces this to $2.49$~cm, but a noticeable residual error remains. In contrast, the proposed Neural-ESO achieves the highest precision, with the smallest mean error of $0.37$~cm and the lowest standard deviation ($0.12$~cm). As shown in Fig.~\ref{fig:hover_pos}, Neural-ESO also demonstrates faster settling and smoother transient response than PD+ESO, indicating the effectiveness of incorporating learning in improving disturbance rejection.}
\subsection{Experiment 2: Landing in OOD Scenario}
\label{sec:exp_trajectory}
To evaluate the learned network's generalization and the robustness of ours under OOD conditions, we introduce an inclined slope near the landing zone (shown in Fig.~\ref{fig:cover} (b)). We hypothesize that this geometric mismatch alters the ground-effect characteristics, and challenges the neural network trained in the normal environment of Experiment 1.

As shown in Fig.~\ref{fig:exp2}, PD+NN~\cite{NeuralLander} exhibits substantially larger oscillations and steady-state error than in Experiment~1, indicating degraded generalization. In contrast, ours still achieves a successful landing: although the neural prior becomes less accurate, the corrective pathway compensates for the resulting residual disturbance. 

\begin{figure}[htbp]
    \centering
    \includegraphics[width=0.85\columnwidth]{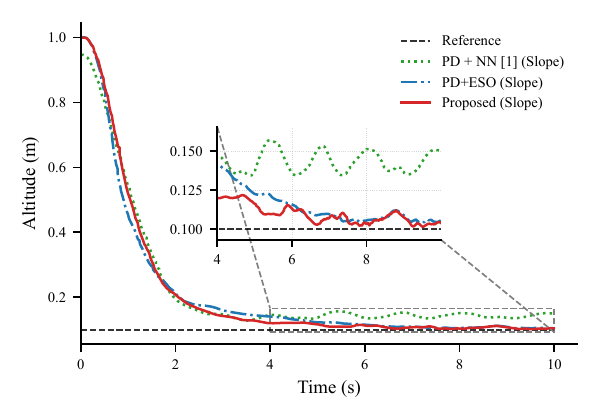}
    \caption{OOD testing with the influence of the slope. PD+NN exhibits much larger oscillations and steady-state error.}
    \label{fig:exp2}
\end{figure}

\begin{table}[htbp]
\centering
\caption{Statistical Analysis of Steady-State Tracking Error\\for Landing in OOD Scenario (5 trials)}
\label{tab:ood_tracking}
\begin{tabular}{lccc}
\toprule
\textbf{Metric} & \textbf{PD+NN [1]} & \textbf{PD+ESO} & \textbf{Proposed} \\
\midrule
Mean [cm] & 3.74 & \textbf{0.66} & 1.53 \\
Std Dev [cm] & 1.88 & \textbf{0.10} & 2.07 \\
\bottomrule
\end{tabular}
\end{table}

{Experiment~2 evaluates OOD robustness using the same model and parameters as in the in-distribution setting. Under distribution shift, the learned prior becomes inaccurate, so stable and safe operation is critical. As shown in Table~\ref{tab:ood_tracking}, PD+NN~\cite{NeuralLander} degrades significantly with large oscillations, while PD+ESO performs best, as it is unaffected by model learning. The proposed method shows higher error ($1.53$~cm) but maintains stable tracking via the corrective pathway, demonstrating robustness to learning errors.}

{The following sensitivity analysis illustrates the trade-off between in-distribution accuracy and OOD robustness governed by the Lipschitz constant. Table~\ref{tab:ood_tracking} only reports results with $L_N=1.0$, which balances high in-distribution accuracy and acceptable robustness. However, a smaller value (e.g., $L_N=0.5$) can further improve OOD performance, slightly surpassing PD+ESO, at the cost of reduced in-distribution accuracy. Experiment~3 will additionally show that retraining in the new domain improves performance beyond PD+ESO.}
\subsection{{Sensitivity Analysis of the Lipschitz Constant}}

\begin{table}[htbp]
\centering
\caption{Ablation Study on Lipschitz Constant $L_N$ across 5 Trials}
\label{tab:lip_change_ablation}
\begin{tabular}{lcccc}
\toprule
\textbf{Metric} & $\mathbf{L_N=0.2}$ & $\mathbf{L_N=0.5}$ & $\mathbf{L_N=1.0}$ & $\mathbf{L_N=2.0}$ \\
\midrule
\multicolumn{5}{l}{\textit{Nominal Setup}} \\
Mean [cm] & 2.73 & 0.71 & \textbf{0.37} & 1.43 \\
Std Dev [cm] & 2.42 & 0.11 & \textbf{0.12} & 2.11 \\
\midrule
\multicolumn{5}{l}{\textit{OOD (Exp.~2)}} \\
Mean [cm] & 1.69 & \textbf{0.63} & 1.53 & 2.59 \\
Std Dev [cm] & 2.01 & \textbf{0.12} & 2.07 & 2.49 \\
\bottomrule
\end{tabular}
\end{table}

{Table~\ref{tab:lip_change_ablation} shows that the Lipschitz constant $L_N$ governs the trade-off between model expressiveness and robustness. A small $L_N$ (\emph{e.g.}, $0.2$) leads to underfitting and poor in-distribution performance, while a large $L_N$ (\emph{e.g.}, $2.0$) causes instability due to excessive sensitivity to noise and extrapolation errors. Importantly, different choices of $L_N$ reflect different design priorities. For example, $L_N=0.5$ achieves the best robustness under OOD conditions (Table~\ref{tab:lip_change_ablation}), while $L_N=1.0$ provides the best in-distribution accuracy ($0.37 \pm 0.12$~cm). In this work, we select $L_N=1.0$ to prioritize high-precision in-distribution tracking while maintaining acceptable robustness under distribution shift.}

{Across all configurations with $L_N < 2.0$, the proposed method consistently outperforms PD+NN, demonstrating the effectiveness of ESO-based compensation in mitigating learning-induced errors.}

\subsection{Experiment 3: {OOD} High-Speed Near-Ground Maneuvers}
The final experiment pushes the Neural-ESO framework to its dynamic limits by executing high-speed maneuvers near the ground. In this scenario, the interaction between the high-velocity downwash and the ground surface transitions from structured ground effects to non-stationary turbulence. This scenario represents a much more severe OOD shift, posing a significant challenge to all compared approaches.

A Lemniscate trajectory was executed in the x-y plane, with $A_x = 0.8$ m, $A_y = 1.0$ m, and $T = 6.0$ s. All the comparisons are shown in Fig. \ref{fig:tracking_perf} and Fig. \ref{fig:tracking_boxplot}. Fig. \ref{fig:tracking_perf} shows comparable tracking performance for the x-y axes, as the same controller is utilized for planar motion across all cases.

\subsubsection{Limitations of Baseline Learning-based Control} {Note that all base neural networks used before retraining are trained on the normal landing case.} As shown in Fig.~\ref{fig:tracking_perf} and Fig.~\ref{fig:tracking_boxplot}, the PD+NN (w/o TDR) exhibits substantial altitude deviations, with tracking error deteriorating to $5.47 \pm 0.56$~cm, worse than the PD controller ($4.08 \pm 1.03$~cm). This highlights poor generalization under severe OOD conditions and indicates that an inaccurate learned prior can degrade performance. In particular, a model trained on low-speed landing data provides little benefit, and may even harm performance, in high-speed near-ground maneuvers. 

\subsubsection{Severe Domain Shift Challenge and Robust Compensation}
Under this extreme domain shift, Neural-ESO without TDR has an RMSE of $4.09 \pm 0.21$~cm, underperforming the PD+ESO baseline ($1.91 \pm 0.83$~cm), which confirms the severity of the domain shift. Nevertheless, it still outperforms PD and PD+NN (w/o TDR), indicating that the ESO continues to compensate for residual disturbances even when the neural prior is inaccurate.

{To further highlight the necessity of this compensation, we compare against a recent neural-observer method \cite{liu2023cascade}, which adopts a single-pathway architecture that replaces the ESO with a neural network, but lacks online ESO-based compensation and any mechanism for further adaptation. Under the same OOD condition, this method fails to maintain stability and cannot complete the flight trajectory, underscoring the safety advantage of our dual-pathway design.}

\subsubsection{Performance Boosting via Total Disturbance Retraining}

To mitigate performance degradation under severe domain shift, we perform offline TDR using newly collected flight data. Specifically, the TDR uses new-domain data collected from 5 OOD flight cycles, totaling 30 seconds of data. Although the total amount of data differs from the nominal case, the 80\%/20\% train/test split and the training pipeline remain unchanged. Note that the corrective pathway enables safe data collection. The total disturbance reconstructed by the corrective pathway ($\hat{f} = f_N + \Delta \hat{f}$) is directly used as the training label to retrain the base network, enabling it to capture the new domain's aerodynamic effects. The performance across $5$ flights is summarized in Fig.~\ref{fig:tracking_boxplot}. With TDR, the proposed framework achieves the lowest RMSE of ${0.89 \pm 0.52}$~cm. Compared to PD+ESO ($1.91 \pm 0.83$~cm), this corresponds to a $53.4\%$ reduction in tracking error.

\begin{figure}[htbp]
    \centering
    \includegraphics[width=0.7\columnwidth]{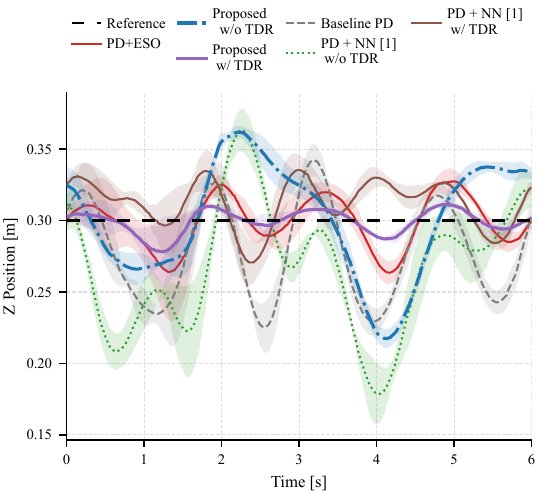}
    \caption{OOD testing under high-speed near-ground maneuvers across five flight cycles. Curves show mean trajectories with shaded $\pm1$ standard deviation regions. For safety, the z-axis reference is set higher than in the previous scenarios.}
    \label{fig:tracking_perf}
\end{figure}

\begin{figure}[htbp]
    \centering
    \includegraphics[width=0.7\columnwidth]{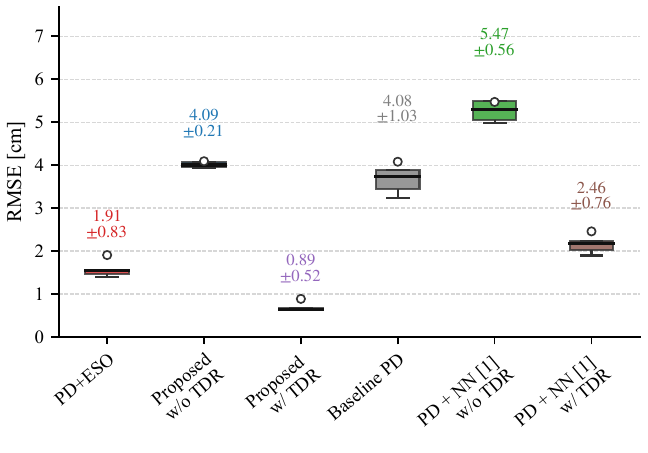}
    \caption{Z-axis tracking RMSE across five flight cycles. Boxes show the interquartile range (IQR), black lines the medians, and white circles the means. Annotations indicate mean $\pm$ standard deviation (cm) for each algorithm across all cycles.}
    \label{fig:tracking_boxplot}
\end{figure}

\section{CONCLUSION}
\label{sec:Conslusion and future work}
This letter presents Neural-ESO, a dual-pathway architecture for disturbance-rejection control that combines a Lipschitz-bounded predictive neural prior with an ESO for online correction. We show a rigorous stability analysis and show that the Lipschitz property of the learned prior provides a sufficient condition to preserve closed-loop robustness in the presence of approximation errors and disturbances. Experiments on a quadrotor UAV under challenging ground-effect disturbances further validate the proposed framework, demonstrating accuracy-robustness trade-off and improved operational reliability across varied conditions, including OOD scenarios, compared with state-of-the-art baselines.
\vspace{-5mm}
\bibliographystyle{IEEEtran}
\bibliography{references}

@article{yu2026parameter,
  title={Parameter-Learning Active Disturbance Rejection Controller for a Novel Variable Stiffness Gripper},
  author={Yu, Ziqing and Fu, Jiaming and Zhang, Fan and Chen, Jinfeng and Gan, Dongming},
  journal={Journal of Mechanisms and Robotics},
  pages={1--11},
  year={2026}
}

@book{khalil2002nonlinear,
  title={Nonlinear Systems},
  author={Khalil, Hassan K and Grizzle, Jessy W},
  volume={3},
  year={2002},
  publisher={Prentice hall Upper Saddle River, NJ}
}

@article{liu2023cascade,
  title={Cascade {ADRC} with neural network-based {ESO} for hypersonic vehicle},
  author={Liu, Lei and Liu, Yongxiong and Zhou, Lilin and Wang, Bo and Cheng, Zhongtao and Fan, Huijin},
  journal={Journal of the Franklin Institute},
  volume={360},
  number={12},
  pages={9115--9138},
  year={2023},
  publisher={Elsevier}
}

@inproceedings{li2023neural,
  title={Neural network learning of robot dynamic uncertainties and observer-based external disturbance estimation for impedance control},
  author={Li, Teng and Badre, Armin and Taghirad, Hamid D and Tavakoli, Mahdi},
  booktitle={2023 IEEE/ASME International Conference on Advanced Intelligent Mechatronics},
  pages={591--597},
  year={2023},
  organization={IEEE}
}

@article{dawson2023safe,
  title={Safe control with learned certificates: A survey of neural {Lyapunov}, barrier, and contraction methods for robotics and control},
  author={Dawson, Charles and Gao, Sicun and Fan, Chuchu},
  journal={IEEE Transactions on Robotics},
  volume={39},
  number={3},
  pages={1749--1767},
  year={2023},
  publisher={IEEE}
}

@inproceedings{he2025self,
  title={Self-supervised meta-learning for all-layer {DNN}-based adaptive control with stability guarantees},
  author={He, Guanqi and Choudhary, Yogita and Shi, Guanya},
  booktitle={2025 IEEE International Conference on Robotics and Automation},
  pages={6012--6018},
  year={2025},
  organization={IEEE}
}

@inproceedings{lee2010geometric,
  title={Geometric tracking control of a quadrotor {UAV} on {SE}(3)},
  author={Lee, Taeyoung and Leok, Melvin and McClamroch, N Harris},
  booktitle={49th IEEE Conference on Decision and Control},
  pages={5420--5425},
  year={2010},
  organization={IEEE}
}

@article{bauersfeld2021neurobem,
  title={NeuroBEM: Hybrid Aerodynamic Quadrotor Model},
  author={Bauersfeld, Leonard and Kaufmann, Elia and Foehn, Philipp and Sun, Sihao and Scaramuzza, Davide},
  journal={Proceedings of Robotics: Science and Systems XVII},
  pages={42},
  year={2021},
  publisher={Robotics: Science and Systems Foundation}
}

@article{chen2024quadrotor,
  title={Quadrotor Fault-Tolerant Control at High Speed: A Model-Based Extended State Observer for Mismatched Disturbance Rejection Approach},
  author={Chen, Jinfeng and Zhang, Fan and Hu, Bin and Lin, Qin},
  journal={IEEE Control Systems Letters},
  year={2024},
  publisher={IEEE}
}

@INPROCEEDINGS{NeuralLander,
  author={Shi, Guanya and Shi, Xichen and O’Connell, Michael and Yu, Rose and Azizzadenesheli, Kamyar and Anandkumar, Animashree and Yue, Yisong and Chung, Soon-Jo},
  booktitle={2019 International Conference on Robotics and Automation}, 
  title={Neural Lander: Stable Drone Landing Control Using Learned Dynamics}, 
  year={2019},
  volume={},
  number={},
  pages={9784-9790},
  doi={10.1109/ICRA.2019.8794351}}

@inproceedings{gao2003scaling,
  title={Scaling and bandwidth-parameterization based controller tuning},
  author={Gao, Zhiqiang},
  booktitle={Proceedings of the 2003 American Control Conference, 2003.},
  pages={4989--4996},
  year={2003},
  organization={IEEE}
}

@inproceedings{zhang2016active,
  title={An active disturbance rejection control solution for the two-mass-spring benchmark problem},
  author={Zhang, Han and Zhao, Shen and Gao, Zhiqiang},
  booktitle={2016 American Control Conference},
  pages={1566--1571},
  year={2016},
  organization={IEEE}
}

@ARTICLE{Han2009,
  author={Han, Jingqing},
  journal={IEEE Transactions on Industrial Electronics}, 
  title={From {PID} to Active Disturbance Rejection Control}, 
  year={2009},
  volume={56},
  number={3},
  pages={900-906},
  doi={10.1109/TIE.2008.2011621}}

@article{brunke2022safe,
  title={Safe learning in robotics: From learning-based control to safe reinforcement learning},
  author={Brunke, Lukas and Greeff, Melissa and Hall, Adam W and Yuan, Zhaocong and Zhou, Siqi and Panerati, Jacopo and Schoellig, Angela P},
  journal={Annual Review of Control, Robotics, and Autonomous Systems},
  volume={5},
  number={1},
  pages={411--444},
  year={2022},
  publisher={Annual Reviews}
}

@article{o2022neural,
  title={Neural-fly enables rapid learning for agile flight in strong winds},
  author={O’Connell, Michael and Shi, Guanya and Shi, Xichen and Azizzadenesheli, Kamyar and Anandkumar, Anima and Yue, Yisong and Chung, Soon-Jo},
  journal={Science Robotics},
  volume={7},
  number={66},
  pages={eabm6597},
  year={2022},
  publisher={American Association for the Advancement of Science}
}

@INPROCEEDINGS{zhang2024,
  author={Zhang, Fan and Chen, Jinfeng and Hu, Yu and Gao, Zhiqiang and Lv, Ge and Lin, Qin},
  booktitle={2024 IEEE 20th International Conference on Automation Science and Engineering}, 
  title={Disturbance Rejection-Guarded Learning for Vibration Suppression of Two-Inertia Systems}, 
  year={2024},
  volume={},
  number={},
  pages={3418-3424},
  doi={10.1109/CASE59546.2024.10711378}}
\end{document}